\documentclass[10pt,twocolumn,letterpaper]{article}

\usepackage[pagenumbers]{cvpr}      

\usepackage{graphicx}
\usepackage{amsmath}
\usepackage{amssymb}
\usepackage{amsfonts}
\usepackage{amsthm}
\usepackage{booktabs,multirow,enumitem,tabularx}
\usepackage{xcolor}
\usepackage{colortbl}
\usepackage{algorithm}%
\usepackage{algpseudocode}
\newtheorem{proposition}{Proposition} 

\definecolor{lightgray}{rgb}{0.95,0.95,0.95}

\usepackage[pagebackref,breaklinks,colorlinks]{hyperref}

\newcommand{\yb}[0]{\textbf{y}}
\newcommand{\xb}[0]{\textbf{x}}

\usepackage[capitalize]{cleveref}
\crefname{section}{Sec.}{Secs.}
\Crefname{section}{Section}{Sections}
\Crefname{table}{Table}{Tables}
\crefname{table}{Tab.}{Tabs.}

\def\cvprPaperID{11405} 
\def\confName{CVPR}
\def\confYear{2023}

\begin{document}

\title{Layer-Stack Temperature Scaling}

\author{Amr Khalifa\thanks{Equal contribution.} \,\thanks{Work performed as part of the AI Residency program at Google. }\,,\; Michael C. Mozer,\; Hanie Sedghi, \;Behnam Neyshabur, \;Ibrahim Alabdulmohsin\footnotemark[1]\\
Google Research, Brain Team\\
{\tt\small \{amrkhalifa,mcmozer,hsedghi,neyshabur,ibomohsin\}@google.com}
}
\maketitle

\begin{abstract}
Recent works demonstrate that early layers in a neural network contain useful information for prediction. Inspired by this, we show that extending temperature scaling across all layers improves both calibration and accuracy. We call this procedure ``layer-stack temperature scaling'' (LATES). Informally, LATES grants each layer a weighted vote during inference. We evaluate it on five popular convolutional neural network architectures both in- and out-of-distribution and observe a consistent improvement over temperature scaling in terms of accuracy, calibration, and AUC. All conclusions are supported by comprehensive statistical analyses. Since LATES neither retrains the architecture nor introduces many more parameters, its advantages can be reaped without requiring additional data beyond what is used in temperature scaling. Finally, we show that combining LATES with Monte Carlo Dropout matches state-of-the-art results on CIFAR10/100.
\end{abstract}

\section{Introduction}
\label{sec:intro}
\paragraph{Motivation.}Recent work in computer vision relates the difficulty of an example to its activations at the \emph{intermediate} layers of the neural network.  For example, when linear classifiers (a.k.a. ``probes'' \cite{alain2016understanding,evci2022head2toe}), are trained on the intermediate activations, it is observed that easy examples tend to be classified correctly by all layers while difficult examples are classified correctly by the last layers alone \cite{cohen2018dnn,baldock2021deep}. In \cite{alain2016understanding}, it is noted that the linear separability of classes increases monotonically with depth, suggesting that early layers might be solving a simpler version of the prediction task. This agrees with other conclusions in early-exit strategies in computer vision \cite{teerapittayanon2016branchynet,huang2017multi} in which less compute is used at inference time for easy examples. Along similar lines, it has been noted when studying transfer learning \cite{raghu2019transfusion,yosinski2014transferable}, memorization \cite{arpit2017closer,cohen2018dnn}, and pretraining with random labels \cite{maennel2020neural} that early layers tend to learn general-purpose representations, which are useful for most examples such as edge detectors, whereas later layers specialize.

The previous observations point to an intriguing hypothesis: does performance (e.g. in terms of calibration and accuracy) improve when \emph{multiple} layers in the neural network vote during inference but not necessarily equally? One recent evidence that supports this hypothesis was demonstrated in the context of transfer learning, in which probing the intermediate layers was found to be a useful alternative to fine-tuning pretrained models \cite{evci2022head2toe}. 

In this work, we introduce ``layer-stack temperature scaling '' (LATES), which extends temperature scaling to intermediate layers in the neural network and aggregates their predictions using the \emph{stacking} procedure \cite{wolpert1992stacked}. Similar to temperature scaling, LATES is a post-hoc method that does not retrain the model nor introduces many more parameters; it only increases the total parameter count by a meager 1\%. Unlike temperature scaling, however, LATES improves accuracy, calibration and area under the ROC curve (AUC). In fact, LATES leads to better calibration even when compared to temperature scaling both in- and out-of-distribution (OOD), particularly in the small-data regime. Moreover, combining LATES with Monte Carlo Dropout \cite{gal2016dropout} matches state-of-the-art results on CIFAR10/100 as reported by the Uncertainty Baselines Benchmark~\cite{nado2021uncertainty}. 

\paragraph{Overview.}Before introducing layer-stack temperature scaling (LATES), it is useful to recall first how temperature scaling \cite{guo2017calibration} works. Temperature scaling is a post-hoc method for calibrating neural networks. It is an extension of the classical Platt scaling algorithm \cite{platt1999probabilistic} to the multiclass setting. It rescales the logits by a scalar $\tau\in\mathbb{R}^+$ chosen according to a separate validation dataset to optimize a \emph{proper} scoring rule, where proper scores are those that are guaranteed to yield a perfectly calibrated predictor at their minima at the infinite data limit \cite{gneiting2007strictly}. Proper scores include, for example, the negative log-likelihood (a.k.a. cross-entropy or log loss) and the Brier score (a.k.a. square loss). Despite its simplicity, temperature scaling yields competitive results in practice for in-distribution data \cite{guo2017calibration,can_u_trust_classifiers,NEURIPS2019_8ca01ea9} but can fail under modest levels of distribution shifts \cite{can_u_trust_classifiers}.

\begin{figure*}
    \centering
    \includegraphics[width=1.65\columnwidth]{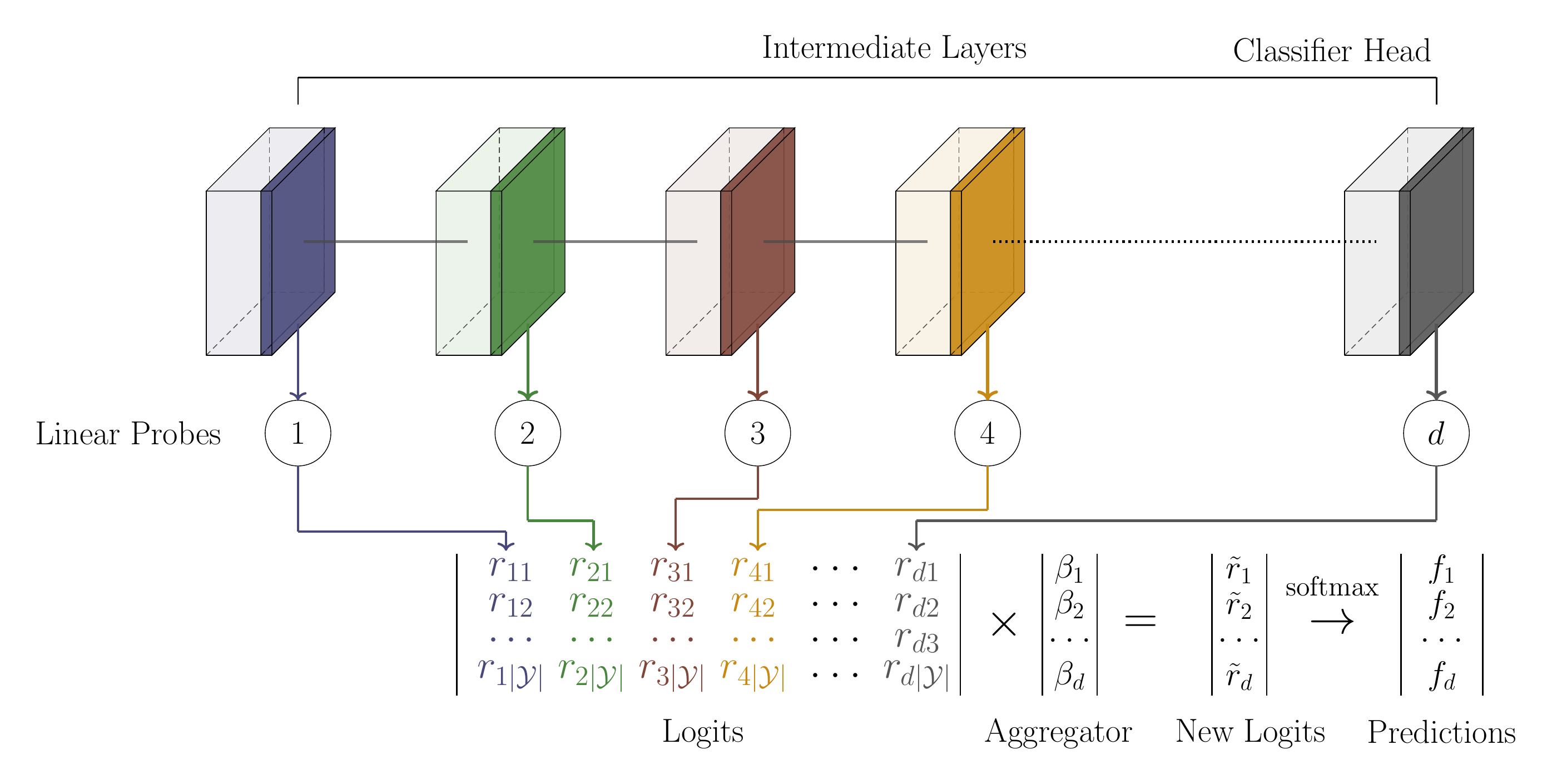}
    \caption{In probe scaling, $d$ linear probes are injected into the intermediate representations of the neural network. Each probe is trained to maximize a proper score, such as the log-loss or square loss, where the loss is chosen according to the desired measure of calibration (e.g. log-likelihood or Brier). Afterward, the logits of all probes are concatenated together to form a new representation. Finally, an aggregator (linear classifier) with $d$ temperature parameters (shown as $\beta$'s) is trained to output the probability scores.}
    \label{fig:probe_scaling}
\end{figure*}

\begin{figure}
    \includegraphics[width=\columnwidth]{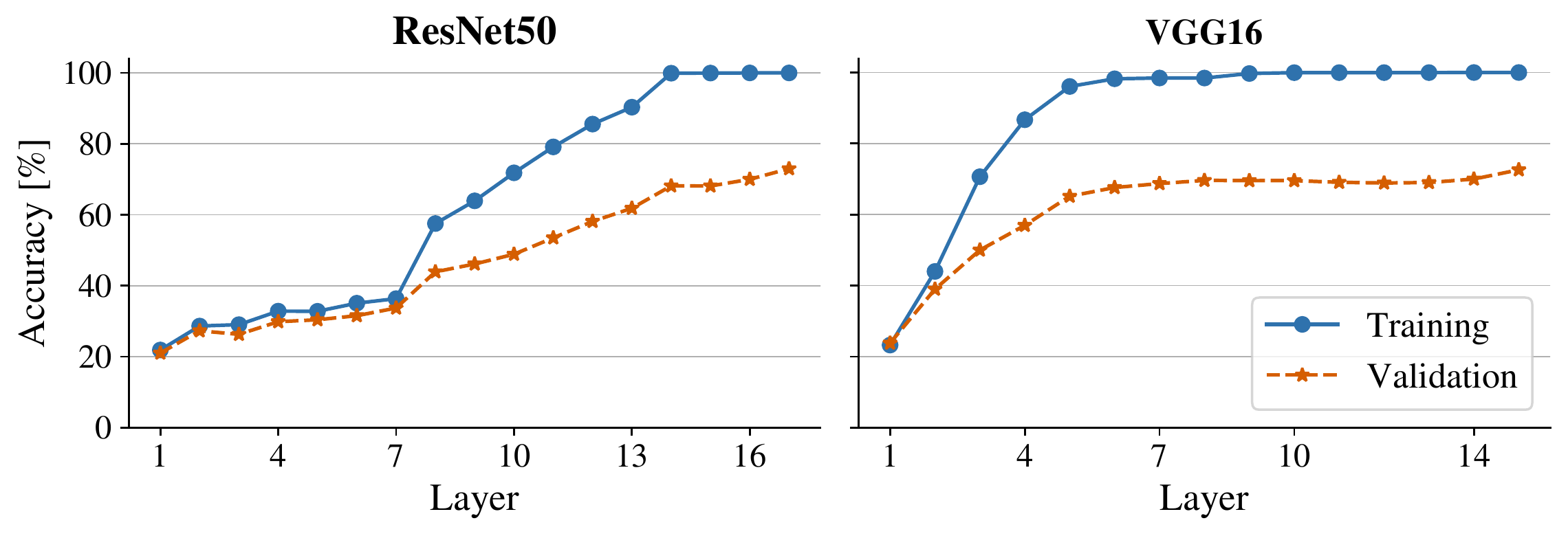}
    \includegraphics[width=\columnwidth]{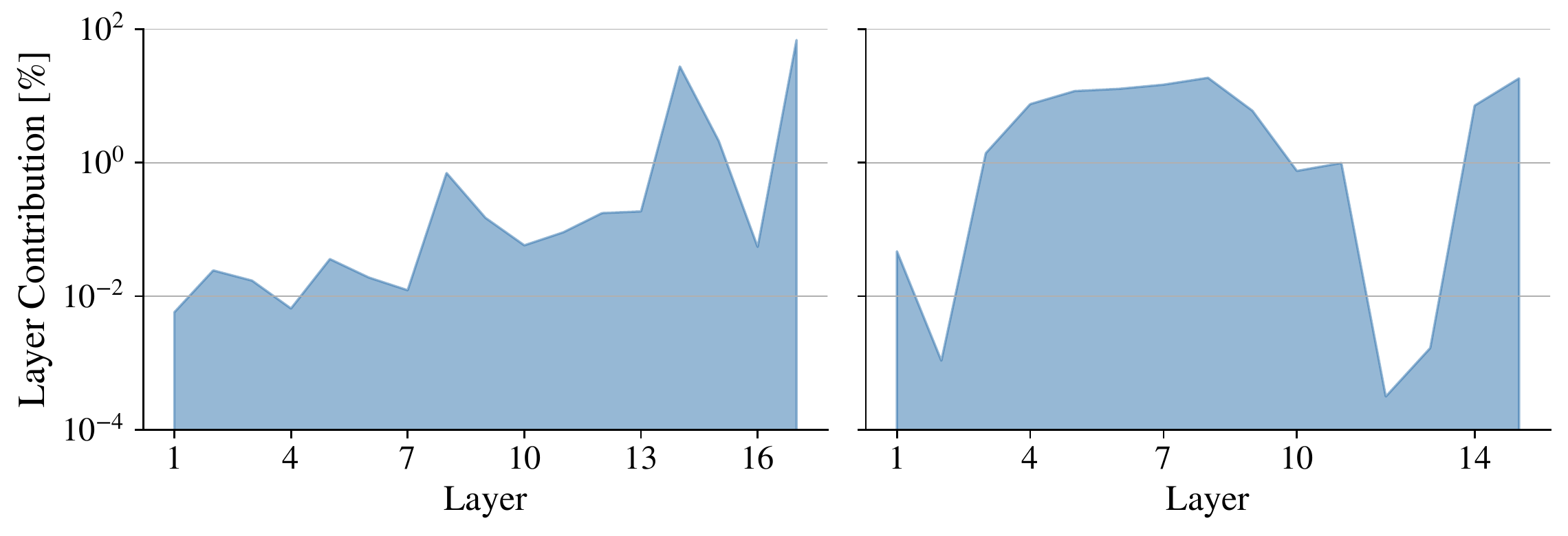}
    \includegraphics[width=\columnwidth]{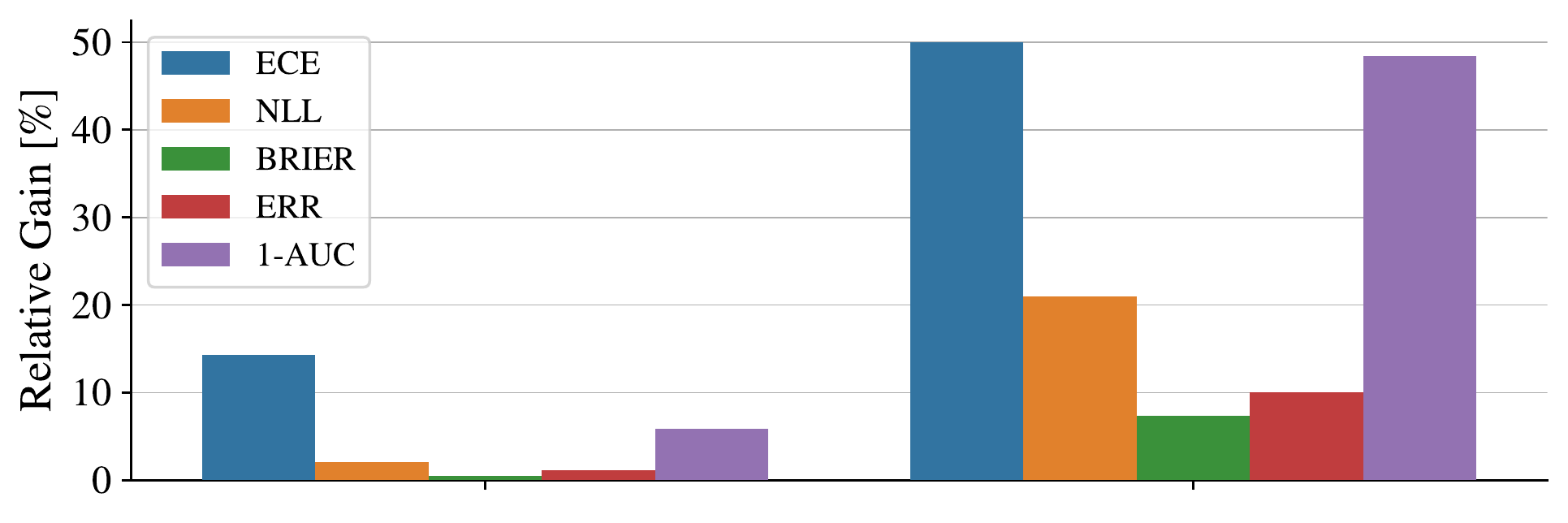}
    \caption{A sample of results using  RESNET50 (left column) and VGG16 (right column) on CIFAR100. {\bf TOP}: accuracy of each probe plotted against its position in the neural network. {\bf MIDDLE}: The contribution of each layer in LATES; see Section \ref{sec:intro} for definition.  {\bf BOTTOM}: the relative gain [\%] (i.e. the percentage of reduction) of LATES compared to temperature scaling. }
    \label{fig:summary}
\end{figure}

In LATES, on the other hand, temperature scaling is extended to \emph{intermediate} layers of the neural network as shown in Figure \ref{fig:probe_scaling}. First, given a deep neural network that can be partitioned into $d$ blocks (e.g. convolutional blocks in ResNet \cite{He2016} and VGG \cite{simonyan2014very}  or encoder blocks in vision transformers (ViT) \cite{dosovitskiy2021vit}), $d$ linear probes are trained on the intermediate representations. Specifically, the $k$-th  linear probe is trained to predict the correct class $\yb\in\mathcal{Y}$ given the activations at block $k$ for each instance $\xb\in\mathcal{X}$. Pooling can be added to reduce the dimensionality of the flattened representation. In total, $d-1$ linear probes are trained while the $d$-th probe is an identity map that preserves the logits of the original classifier. Finally, the logits of all probes are concatenated to form a new representation in the logit space with $|\mathcal{Y}|\times d$ features as shown in Figure \ref{fig:probe_scaling}.  Full pseudocode is provided in Algorithm \ref{algorithm}.

\paragraph{Relation to Temperature Scaling.}Compared to temperature scaling, LATES improves \emph{both} accuracy and calibration as demonstrated in Figure \ref{fig:summary}. However, calibration can be measured in various ways. Depending on the desired measure of calibration, a corresponding loss function is chosen to train the aggregator in LATES. For example, if the desired measure of calibration is the negative-log-likelihood (NLL), the cross-entropy loss is minimized. By contrast, the square loss is minimized if calibration is measured using the Brier score \cite{brier1950verification}. Soft differentiable losses also exist for the expected calibration error (ECE) \cite{karandikar2021soft}. 

Clearly, LATES reduces to temperature scaling when $\beta_k=0$ for all $k<d$. In Section \ref{sect::theory}, we prove that temperature scaling cannot perform better than LATES with a high probability when evaluated on the loss being optimized.  Experimentally, LATES often performs \emph{strictly} better, including on metrics it is not optimized for; e.g. both ECE and AUC improve when optimizing for NLL. This is true even though LATES does \emph{not} require additional data beyond what is often used in temperature scaling. Note that the number of parameters used by the aggregator in Figure \ref{fig:probe_scaling} is quite small; e.g. does not exceed 20 parameters in all of our experiments. Thus, the same validation split used for temperature scaling is sufficient to train the aggregator as well\footnote{The linear probes are trained on the original training split so they do not require additional data, either.}. We provide further analysis of this in Section \ref{sect::exp_in_d}.

We emphasize that the improvement in LATES happens even when the original model is trained until convergence. Figure \ref{fig:summary}~({\sc top}) displays the accuracy of each probe for ResNet50 and VGG16  trained on CIFAR100 with left/right augmentation. Figure \ref{fig:summary}~({\sc middle}) highlights the contributions of each layer learned during stacking in LATES, where the contribution of layer $k$ is defined by $\beta_k/\sum_i\beta_i$. Clearly, intermediate layers contain useful information to improve performance. A key takeaway is that the original model achieves perfect training accuracy and the last layer achieves the highest validation accuracy, yet LATES improves the ECE and log-loss compared to temperature scaling as shown in Figure \ref{fig:summary} ({\sc bottom}). 

In this work, we compare LATES against temperature scaling in CNNs. We use VGG16 and VGG19 \cite{simonyan2014very} as a sample of architectures without residual links, and use ResNet18, ResNet50 \cite{He2016}, and Wide-ResNet-28-10 \cite{zagoruyko2016wide} as a sample of architectures with residual links. We follow the Uncertainty Baselines Benchmark~\cite{nado2021uncertainty} and focus our analysis on CIFAR10/100 with several corruption types added during inference of various severity levels \cite{Krizhevsky09learningmultiple}. We use ImageNet-ILSVRC2012 \cite{deng2009imagenet} and ResNet50 \cite{He2016} to assess if such conclusions carry over to larger datasets.

\paragraph{Parameter Count and Time.}
The total compute overhead LATES contributes to inference on the network is less than 0.7\% (in GFlops), measured on ResNet50. Training all probes is  parallelizable (all probes are trained together as a single model) and adds, on average, less than 25\% more training time. In terms of parameter count, the total number of parameters introduced by LATES does not exceed 1\% of the original architecture.

\paragraph{Statement of Contribution.}We summarize, next, our contribution and findings. First, we introduce and theoretically motivate layer-stack temperature scaling (LATES) as an extension of temperature scaling to all layers in the neural network. It outperforms temperature scaling both in- and out-of-distribution (OOD) in terms of accuracy, calibration, and AUC, particularly in the small-data regime and OOD. Second, we support all of our conclusions by a comprehensive statistical analysis (see Section \ref{sect:stat}). Third, we show that combining LATES with Monte Carlo Dropout \cite{gal2016dropout} matches state-of-the-art results on CIFAR10/100.

\paragraph{Summary of Notation.}We use boldface letters \textbf{x} for random variables, small letters $x$ for non-random numbers (e.g. constants or instances or random variables), capital letters $X$ for matrices, and calligraphic typeface $\mathcal{X}$ for sets. We write $[N]$ for the set $\{0, 1, 2, \ldots, N-1\}$ and denote the probability simplex in $\mathbb{R}^d$ by $\Delta_{d-1}$.
\begin{algorithm}[tb]
 \small
 \vspace{1mm}
 \textbf{Input:} (1) Trained neural network with a chosen sequence of $d-1\ge 0$ blocks; (2) Hold-out data; (3) loss  $l:\Delta_{K-1}\times \mathcal{Y}\to\mathbb{R}$.
 \vspace{1mm}
 
 \textbf{Output:} Calibrated model.
 \vspace{1mm}
 
 \textbf{Training:}
 \vspace{1mm}
    \begin{algorithmic}[1]
     \State Insert $d-1$ linear probes, where probe $k$ is trained to predict the target $\yb\in\mathcal{Y}$ given the activations in block $k$.
     \State Concatenate the logits of all $d-1$ probes in addition to the logits of the original neural network, forming a new representation with $|\mathcal{Y}|\times d$ features. Let:
     \begin{equation*}
         R(\xb) = \begin{bmatrix}
         r_{1,1} & r_{2,1} & \cdots & r_{d,1}\\
         r_{1,2} & r_{2,2} & \cdots & r_{d,2}\\
         \cdots & \cdots & \cdots & \cdots \\
         r_{1,|\mathcal{Y}|} & r_{2,|\mathcal{Y}|} & \cdots & r_{d,|\mathcal{Y}|}\\
         \end{bmatrix}
     \end{equation*}
     be the new representation, where $r_{i,y}$ is the logit of the $i$-th probe for the class $y$.
     \State Initialize $\beta_0 = (0,0,\ldots,0,1)^T\in\mathbb{R}^d$. 
     \State
     Train an aggregator using the projected stochastic gradient descent (SGD) that minimizes:
     \begin{equation*}
         \mathbb{E}_{\xb,\yb} \Big[l\big(\text{softmax}\big(R(\xb)\,\beta\big),\;\yb\big)\Big],
     \end{equation*}
     over the $d$ parameters $\beta\in\mathbb{R}^d$ subject to the constraint $\beta\ge 0$, and starting from the initial value $\beta_0$.\vspace{1em}
    \end{algorithmic}
\caption{Pseudocode of Probe Scaling}\label{algorithm}
\end{algorithm}

\section{Preliminaries and Related Work}\label{sect:related}

\paragraph{Preliminaries.}Besides classification accuracy, improving calibration is arguably equally important. Calibration is a measure of how accurate a model is in estimating its own uncertainty. It is an important metric for evaluating predictive models, particularly in critical domains, such as medical diagnosis, where decisions depend on both the model's confidence and the misclassification cost. They are also useful in \emph{selective predictions} (a.k.a. reject option classifiers); e.g. by referring to consultants for second opinions or by supplying additional information such as lab tests, among others \cite{kompa2021second,geifman2017selective,vickers2016net}. In addition, calibration serves an essential role when debiasing trained models via post-processing, in which group-specific thresholding rules are adjusted to maximize accuracy subject to prescribed fairness guarantees \cite{corbett2017algorithmic,menon2018cost,celis2019classification,alabdulmohsin2021near}.

Informally, a model is said to be \emph{calibrated} if its predicted probabilities closely agree with the observed empirical frequencies. Formally speaking, writing $\mathcal{X}$ for the instance space (e.g. images), $\mathcal{Y}=[K]$ for the target set, and $\Delta_{K-1}$ for the probability simplex in $\mathbb{R}^K$, a probabilistic multiclass classifier $f:\mathcal{X}\to\Delta_{K-1}$ is calibrated if \cite{NEURIPS2019_8ca01ea9}:
\begin{equation}\label{eq::calibraiton_def}
    \forall y\in\mathcal{Y}: \; p(\yb = y \,| f(\xb)=q) = q_y,
\end{equation}
where $q_y$ is the $y$-th coordinate of the probability distribution $q\in\Delta_{K-1}$. 
Often, this definition is relaxed to focus on the model's top prediction alone \cite{naeini2015,nixon2019measuring,guo2017calibration,can_u_trust_classifiers} (see Section \ref{sect:related}).

Unfortunately, several lines of work observe that modern neural networks lack such calibration \cite{guo2017calibration,thulasidasan2019mixup,can_u_trust_classifiers,kumar2019calibration,NEURIPS2019_8ca01ea9}, particularly under distribution shift \cite{can_u_trust_classifiers}. To remedy this problem, various approaches have been proposed including Monte Carlo dropout \cite{gal2016dropout}, latent Gaussian processes \cite{wenger2020non}, deep ensembles \cite{lakshminarayanan2016simple}, training multiple independent subnetworks \cite{havasi2020training}, and augmentation \cite{thulasidasan2019mixup}. A typical approach in these methods is to provide \emph{diversity} by generating multiple predictions for the same instance $\xb\in\mathcal{X}$. In addition, there are regularization-based methods, such as the Maximum Mean Calibration Error (MMCE), which adds a penalty term to the loss that is analogous to the Maximum Mean Discrepancy (MMD) method  \cite{kumar2018trainable}. Moreover, there are Bayesian methods, such as Stochastic Weight Average Gaussian (SWAG), which estimates uncertainty by sampling  from the posterior distribution of the weights \cite{maddox2019simple}. Finally, several post-hoc methods have been proposed, such as isotonic regression \cite{zadrozny2002transforming}, which learns a monotonic transformation of the model's predictions, the splines-based approach in \cite{gupta2020calibration}, which approximates cumulative distribution functions using splines and estimates uncertainty using derivatives, and the histogram binning approach in \cite{patel2020multi}, which calibrates by maximizing the mutual information. 

In \cite{liu2020simple}, the spectral normalized Gaussian process method is introduced, in which the classifier's head is replaced with a Gaussian process and Augmentation Mixup (AugMix) is used. It achieves state-of-the-art results on CIFAR10 and CIFAR100. We show that by combining LATES with Monte Carlo Dropout \cite{gal2016dropout}, similar results are achieved even with very limited augmentation (only left/right flipping). 

\paragraph{Measures of Calibration}
To evaluate calibration, several scores have been introduced in the literature. As mentioned earlier, the choice of the loss in LATES can depend on the desired measure of calibration, such as by minimizing the log-loss if calibration is measured using the negative log-likelihood (NLL). We describe three popular options, next. 

One common score for calibration is the expected calibration error (ECE), which in its confidence-based version is given by \cite{naeini2015,guo2017calibration,nixon2019measuring}:
\begin{equation}
    \mathrm{ECE} = \sum_{j=0}^{m-1} \frac{|B_j|}{n}\,\big|\mathrm{acc}(B_j) \,-\,p(B_j)\big|,
\end{equation}
where $\{B_j\}_{j\in[m]}$ are the bins, $|B_j|$ is the size of the bin, $acc(B_j)$ is the model's accuracy within the bin, and $p(B_j)$ is the model's average confidence. To optimize ECE in LATES, differentiable approximations can be used \cite{karandikar2021soft}.

While popular, ECE has shortcomings. It is  sensitive to the number of chosen bins \cite{nixon2019measuring} and biased \cite{kumar2019verified}. In addition, it omits important performance indicators that are captured by proper scoring rules. In \cite{kull2015novel}, proper scoring rules are shown to decomposed into a sum of three terms: calibration, group loss, and irreducible loss. ECE captures the first term \emph{alone}, thus permitting a discrepancy between the calibration probabilities and the true posterior probabilities even when the model enjoys a perfect accuracy and ECE \cite{kull2015novel}. We illustrate this pitfall of ECE in Figure \ref{fig:ece_fails}.

\begin{figure}
    \centering
    \includegraphics[width=0.9\columnwidth]{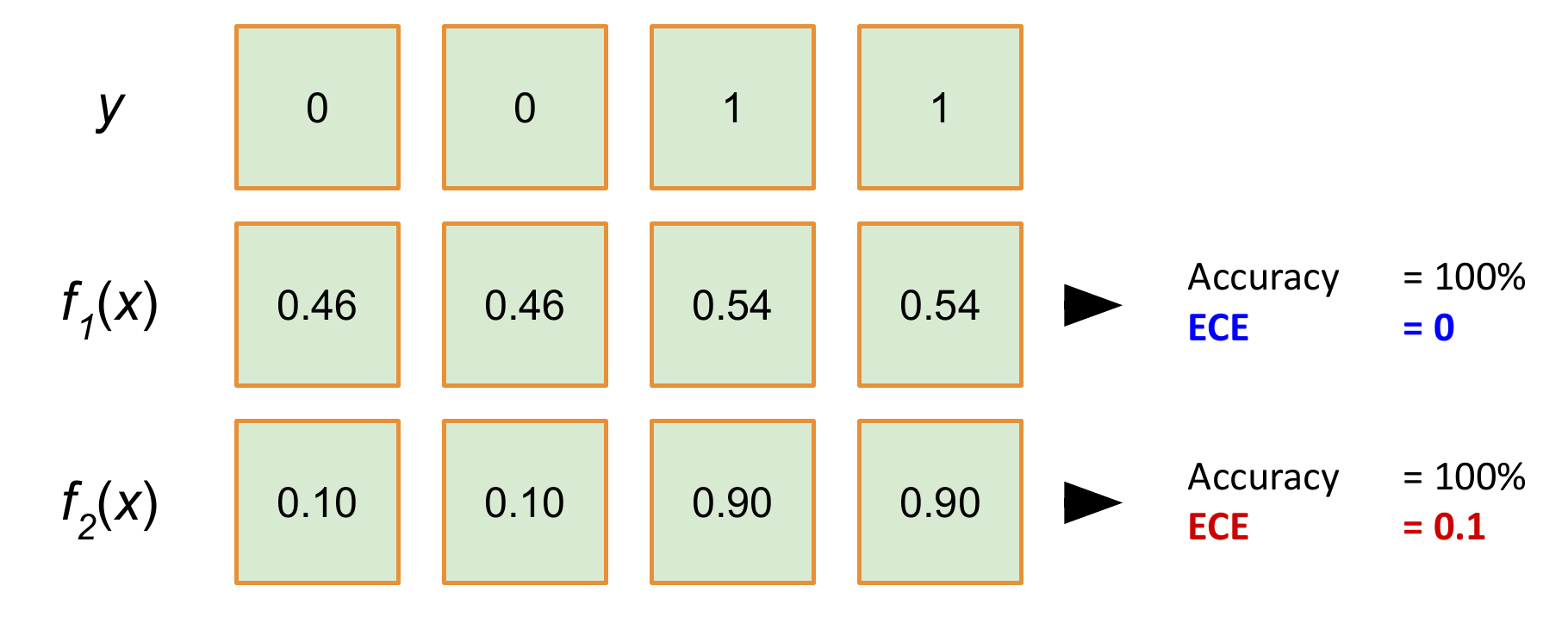}
    \caption{An illustration of one of the pitfalls of using ECE with equal-width binning even when combined with accuracy. In this example, we have four training data points with true labels $\{0, 0, 1, 1\}$ and two predictors $f_1(x)$ and $f_2(x)$, which make the predictions $p(\yb=+1|\xb)$ shown in the middle and bottom rows respectively. Even though $f_2(x)$ has a perfect accuracy and is clearly better than $f_1(x)$, it performs worse than $f_1(x)$ according to ECE with ten equal-width bins.}
    \label{fig:ece_fails}
\end{figure}

To circumvent the limitations of ECE, proper scores can be used. As discussed earlier, a scoring function is called proper if it is optimized by the true probability distribution, and it is called \emph{strictly proper} if it has a unique optima \cite{gneiting2007strictly}. One example is the Brier score \cite{brier1950verification}:
\begin{equation}\label{eq:brierscore}
    \mathrm{Brier} = \mathbb{E}_{\xb,\yb}\big[\sum_{y\in\mathcal{Y}}p(y\,|\,\xb)^2 \,-\,2p(\yb\,|\,\xb)\big],
\end{equation}
where $\yb$ is the true label and $p(y\,|\,\xb)$ is the probability assigned by the model to the class $y\in\mathcal{Y}$. Note that the Brier score is an affine transformation of the square loss and can be negative when Equation \ref{eq:brierscore} is used. In LATES, the Brier score can be optimized by minimizing the square loss. 

Besides, the negative log-likelihood (NLL), a.k.a. cross entropy or log-loss, is a second proper score, which assigns a more aggressive penalty than the Brier score for events that occur when their predicted probabilities were quite small \cite{bickel2007some}. It is optimized in LATES by minimizing the cross-entropy loss. In our evaluation, we follow this approach since this corresponds to what is commonly used in temperature scaling. While this emphasizes on the log-loss, we also report ECE, Brier, accuracy, and AUC. 

\paragraph{Probes.}Probes were introduced in \cite{alain2016understanding} to study the intermediate representations of neural networks, and were subsequently used in several works \cite{evci2022head2toe,alain2016understanding,baldock2021deep,cohen2018dnn}. Originally, the aim of having such probes was to improve our understanding of neural networks. Our contribution is different in that we show how linear probes provide a means of extending temperature scaling across intermediate layers of the neural network to improve both calibration and accuracy. 

Several works apply intermediate classifiers after various neural network layers. This includes, for example, Deeply-Supervised Nets \cite{lee2015deeply} and anytime classification networks \cite{bolukbasi2017adaptive,huang2017multi,teerapittayanon2016branchynet}. However, these works do not use intermediate classifiers to improve calibration and accuracy.

\section{Theoretical Justification}\label{sect::theory}
In this section, we present a formal theoretical justification for LATES. To simplify exposition, we focus on the binary classification setting. The following theorem establishes the advantage of LATES over temperature scaling when LATES is trained to optimize a proper, convex, Lipschitz scoring function, such as the log-loss in bounded domains. As stated earlier, the choice of the loss optimized in LATES can depend on the desired measure of calibration.

\begin{proposition}\label{prop}
Let $f:\mathbb{R}^{d}\times\mathcal{X}\times\mathcal{Y}\to\mathbb{R}$ be a scoring function that is convex and $\rho$-Lipschitz continuous on its first argument. For a given domain $\mathcal{X}$ and a target set $\mathcal{Y}$, denote by $\textbf{S}\in(\mathcal{X}\times\mathcal{Y})^n\sim\mathcal{D}^n$ a hold-out dataset of size $n$. Let $\beta_{L}$ be the aggregator weights in LATES that minimize the regularized empirical loss on the hold-out sample:
\begin{equation}
    \beta_{L} = \arg\min_{\beta\in\mathbb{R}^d} \mathbb{E}_{\xb,\yb \in \textbf{S}}\big[(\lambda/2)\,||\beta||^2 + f(\beta,\,\xb,\,\yb)\big],
\end{equation}
Let $\beta^\star$ be the optimal LATES parameters at the limit $n\to\infty$. Then, with probability of at most $\delta \le \frac{1}{\epsilon}\big(\lambda||\beta^\star||^2 + \frac{2\rho^2}{\lambda n}\big)$ over the choice of \textbf{S}, one has for any weights $\beta$:
\begin{equation}
\mathbb{E}_{\xb,\yb\in\mathcal{D}}\big[f(\beta, \xb, \yb)\big] \le \mathbb{E}_{\xb,\yb\in\mathcal{D}}\big[f(\beta_{L}, \xb, \yb)\big] - \epsilon.
\end{equation} 
\end{proposition}
\begin{proof}
The proof of this proposition uses Corollary 13.8 in \cite{shalev2014understanding}. Let $L_S(\beta)$ be the empirical loss w.r.t. $f$ and write $L_{\mathcal{D}}(\beta)$ for the population loss of $\beta$. If $f$ is a convex function that is $\rho$-Lipschitz continuous, then the solution to the regularized empirical risk minimization problem:
\begin{equation*}
    \hat{\beta} = \arg\min_{w} \big\{(\lambda/2)||w||^2 + L_S(w)\big\},
\end{equation*}
satisfies the oracle inequality:
\begin{equation*}
    \forall {w}:\quad \mathbb{E}_S\big[L_{\mathcal{D}}(\hat{\beta}\big] \le L_{\mathcal{D}}(w) + \lambda ||w||^2 + \frac{2\rho^2}{\lambda n}.
\end{equation*}
Choosing $w = \beta^\star \doteq \arg\min_w\{L_{\mathcal{D}}(w)\}$:
\begin{equation*}
\mathbb{E}_S\big[L_{\mathcal{D}}(\hat{\beta})\big]  \le L_{\mathcal{D}}(\beta^\star) + \lambda ||\beta^\star||^2 + \frac{2\rho^2}{\lambda n}.
\end{equation*}

Using the definition of $\beta^\star$ and applying Markov's inequality, we have for any $\beta$:
\begin{align*}
    p_S\Big\{L_{\mathcal{D}}(\hat{\beta}) - L_{\mathcal{D}}(\beta) \ge \epsilon\Big\} &\le 
    p_S\Big\{L_{\mathcal{D}}(\hat{\beta}) - L_{\mathcal{D}}(\beta^\star) \ge \epsilon\Big\} \\
    &\le \frac{1}{\epsilon} \Big( \lambda\,||\beta^\star||^2 + \frac{2\rho^2}{\lambda n}\Big)
\end{align*}

Setting $\hat{\beta}$ for the solution learned by probe scaling and $\beta$ for temperature scaling (which is valid because temperature scaling belongs to the search space of probe scaling), we obtain the statement of the proposition.
\end{proof}

One implication of Proposition \ref{prop} is that by setting $\lambda = \Theta(1/\sqrt{n})$, the probability bound $(||\beta^\star||^2 + 2\rho^2)/(\epsilon\sqrt{n})$ can be made arbitrarily small for a sufficiently large $n$ for any fixed $\epsilon>0$. Since temperature scaling corresponds to the particular choice of $\beta_k=0$ for all $k<d$, where $d$ is the total number of probes, temperature scaling cannot perform better than probe scaling with a high probability.

\section{Empirical Evaluation}\label{sect::exp_in_d}

\begin{figure*}
    \centering
    \includegraphics[width=2\columnwidth]{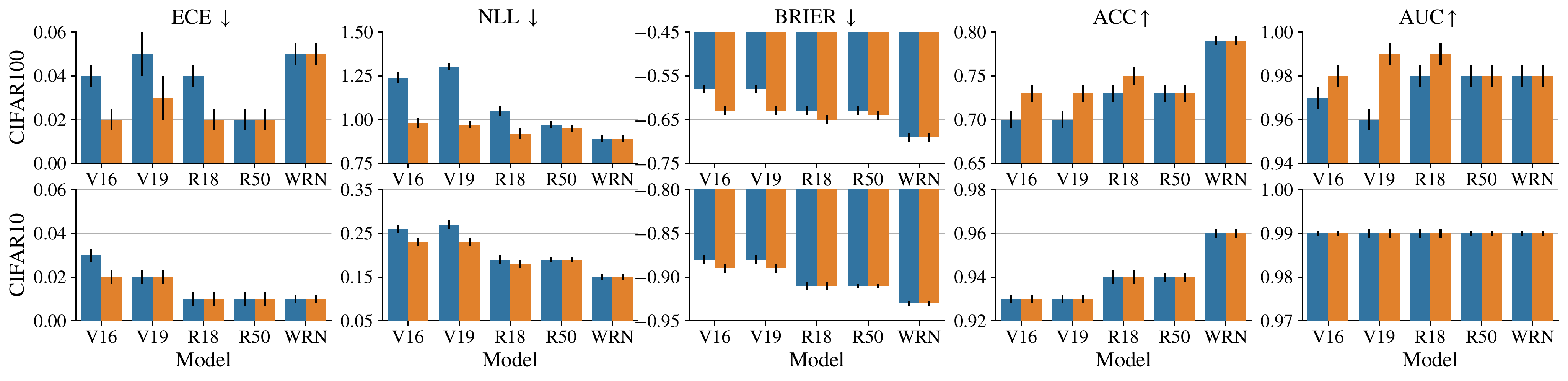}
    \caption{A comparison between LATES (in {\color{brown}\bf amber}) and temperature scaling (in {\color{blue}\bf blue}) on both CIFAR100 (top) and CIFAR10 (bottom) \cite{Krizhevsky09learningmultiple} for the five metrics: ECE, negative log-likelihood (NLL), Brier score, accuracy (ACC), and AUC. The five architectures are: {\sc vgg16} ({\sc v16}), {\sc vgg19} ({\sc v19}), {\sc resnet18} ({\sc r18}), {\sc resnet50} ({\sc r50}), and {\sc wide-resnet-28-10} ({\sc wrn}).}
    \label{fig:cifarind}
\end{figure*}

\subsection{Experiment Setup}
\paragraph{Datasets and Architectures.}
In our experiments, we follow the Uncertainty Baselines Benchmark~\cite{nado2021uncertainty} and focus our analysis on CIFAR10/100 \cite{Krizhevsky09learningmultiple} and ImageNet-ILSVRC2012 \cite{deng2009imagenet}  with several corruption types added during inference of various severity levels. We use five architectures: VGG16 and VGG19 \cite{simonyan2014very}, ResNet18 and ResNet50 \cite{He2016}, and Wide-ResNet-28-10 \cite{zagoruyko2016wide} to examine the impact of both architecture size and residual links in convolutional neural networks. Examples of corruption types include Gaussian noise, JPEG compression, spattering, and blurring, among others. See \cite{nado2021uncertainty} for details.

\paragraph{Training Setup.}To calibrate neural networks using either temperature scaling or LATES, we split the data into three parts: (1) \emph{training}, 90\% of standard training split, (2) \emph{hold-out}, 10\% of standard training split, and (3) \emph{test}. The original architecture and probes are trained on the training split while the calibration rule is trained on the hold-out data.

We train the baseline model with $\ell_{2}$ regularization of $10^{-4}$ for 200 epochs on CIFAR10/100, with a batch size of 128 using SGD with learning rate 0.1 and momentum 0.9. We decay the learning rate using cosine scheduling \cite{loshchilov2016sgdr}. In ImageNet-ILSVRC2012, we train for 90 epochs in addition to a 5-epoch warm-up phase using a batch size of 256 with initial learning rate of 0.1 with $10\times$ decay at 30\%, 60\%, and 90\% of the training steps. On the input pipeline, we standardize the input to have values in the unit interval $[0,\,1]$ and use only right-left random flipping as data augmentation. We train on NVIDIA Tesla V100 GPUs for CIFAR10/100 and on TPUs for ImageNet-ILSVRC2012.

In LATES, we flatten the intermediate activations and reduce dimensionality using average pooling. We place a probe after every convolutional block in ResNet and Wide-ResNet architectures, and after every convolutional layer in VGG16/19. This results in 8 probes in ResNet18, 16 probes in ResNet50, 12 probes in Wide-ResNet-28-10, 14 probes in VGG16, and 17 probes in VGG19. We train each probe with a learning rate of 0.01 and momentum 0.9 with $2\times$ decay after every 10 epochs. All probes are trained together (in parallel) for 50 epochs,  although we observe that they converge much faster. Finally, we train the aggregator model for 50 epochs with a fixed learning rate of 0.005, using the hold-out part of the dataset. We repeat all experiments five times and report averages.

\subsection{In-Distribution Data}
For in-distribution CIFAR10/100, experimental results are summarized in Figure \ref{fig:cifarind}. We observe a consistent improvement in LATES over temperature scaling across all metrics. In CIFAR10, while the accuracy does not improve, LATES improves calibration. In addition, the advantage of LATES compared to temperature scaling becomes bigger  in the low-data regime, as demonstrated in Figure \ref{fig:lowdata}.

In ImageNet-ILSVRC2012, we compare LATES against temperature scaling using RESNET50 to evaluate whether the improvement of LATES extends to much bigger datasets. As shown in Table \ref{table:ind-imagenet}, we observe a significant improvement in calibration in LATES (e.g. ECE is reduced from 0.06 to 0.023 and NLL is reduced from 1.16 to 1.04). In addition, LATES boosts accuracy as well.

\begin{table}
\scriptsize
\caption{A comparison between LATES (abbreviated L), temperature scaling (abbreviated T) on ImageNet-ILSVRC2012 \cite{deng2009imagenet} \cite{Krizhevsky09learningmultiple}. We observe a significant improvement in calibration while also boosting accuracy.}\vspace{1mm}
\label{table:ind-imagenet}
\begin{tabularx}{\columnwidth}{@{}llXXXXX@{\hspace{0.1em}}}
\toprule
&&
  \multicolumn{1}{l}{ECE $\downarrow$} &
  \multicolumn{1}{l}{NLL $\downarrow$} &
  \multicolumn{1}{l}{Brier $\downarrow$} &
  \multicolumn{1}{l}{ACC $\uparrow$} &
  \multicolumn{1}{l}{AUC $\uparrow$}
    \\ \midrule
    \phantom{v000} &
  \begin{tabular}[l]{@{}l@{}} \sc t \\ \bf \sc l
    \end{tabular} & \begin{tabular}[c]{@{}c@{}}   0.060 \\  \bf 0.023 \end{tabular}
    &\begin{tabular}[c]{@{}c@{}}   1.156 \\   \bf 1.039 \end{tabular}
    &\begin{tabular}[c]{@{}c@{}}   -0.629 \\ \bf -0.644 \end{tabular}
    &\begin{tabular}[c]{@{}c@{}}   0.738 \\  \bf 0.745 \end{tabular}
    &\begin{tabular}[c]{@{}c@{}}  \bf 0.991  \\ \bf 0.991 \end{tabular} \\ \bottomrule
\end{tabularx}
\end{table}

\begin{figure}
    \centering
    \includegraphics[width=\columnwidth]{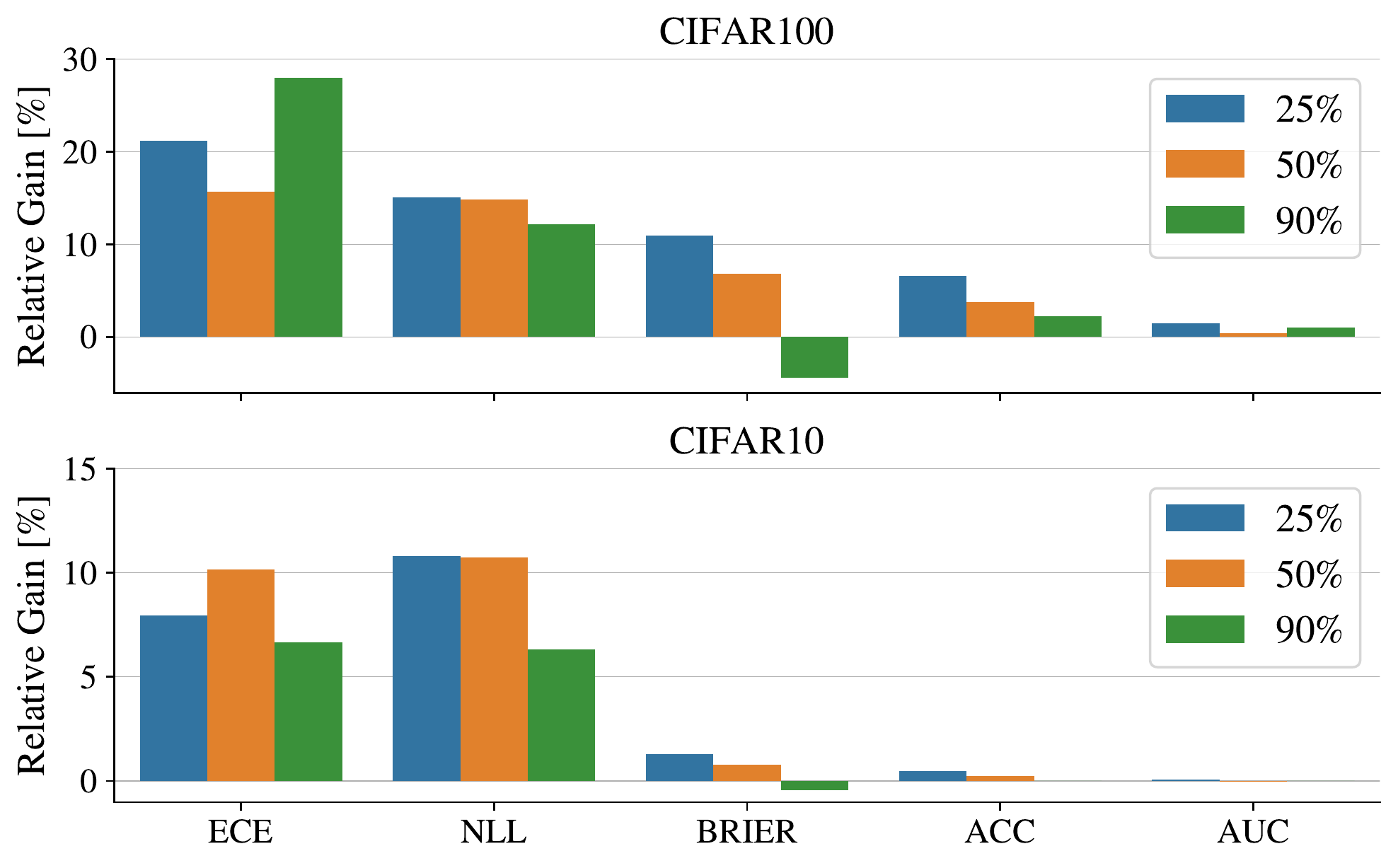}
    \caption{A summary of the relative gain of LATES compared to temperature scaling when the model is trained on 25\%, 50\%, and 90\% of the standard training splits while the temperatures (in both temperature scaling and LATES) are trained on the same 10\% hold-out dataset. Here, ``relative gain'' is the percentage-wise improvement in the metric. Negative values indicate a drop in performance compared to temperature scaling. In general, the advantage of LATES tends to increase for small data sizes.}
    \label{fig:lowdata}
\end{figure}

Nevertheless, only a single temperature parameter is trained in temperature scaling. By contrast, LATES trains multiple temperatures  (one for each probe). We use this observation to evaluate if temperature scaling can yield better results when more data is used when training the original model. We vary the hold-out data split to make up $(2\%, 4\%, 6\%, 8\%, 10\%)$ of the standard training split in both CIFAR10 and CIFAR100 and compare performance across all metrics.  We do not observe any improvement in temperature scaling in the latter case. In fact, calibration tends to get \emph{worse} as expected, since the temperature parameter is now trained on less data. For example, a reduction in the size of the hold-out dataset increases the ECE in temperature scaling in about 65\% of the cases. On the other hand, LATES improves \emph{both} accuracy and calibration.


\begin{table*}
\scriptsize
\centering
\caption{Statistical significance tests using Analysis of Variance. Unlike the non-parametric signed rank test in Table \ref{table:stat}, ANOVA does not report statistical significance in CIFAR100 in terms of ECE, and agrees with the signed ranked test in all other cases. }\vspace{1mm}\label{table:anova}
\begin{tabularx}{2\columnwidth}{@{}XXXXXX@{\hspace{0.1em}}}
\toprule
Dataset
& ECE & NLL & Brier & ACC & AUC
    \\ \midrule
 CIFAR100
  & $F(1,16)=2.31$, 
    & $F(1,16)=63.4$ & $F(1,16)=25.6$  &$F(1,16)=108.7$  & $F(1,16)=29.1$  \\
    & $p=0.148 $ & $p<.001$ & $p<.001$ & $p=.001$ & $p<.001$\\ \midrule
 CIFAR10 
  & $F(1,16)=70.5$, 
    & $F(1,16)=39.5$ & $F(1,16)=38.8$  &$F(1,16)=7.72$  & $F(1,16)=22.6$  \\
    & $p<.001$ & $p<.001$ & $p<.001$ & $p=.013$ & $p<.001$\\ \midrule
 ImageNet 
  & $F(1,16)=62.8$, 
    & $F(1,16)=84.6$ & $F(1,16)=41.52$  &$F(1,16)<1$  & $F(1,16)=56.6$  \\
    & $p=.001$ & $p<.001$ & $p<.001$ & $p=.57$ & $p<.001$
    \\ \bottomrule
\end{tabularx}
\end{table*}

\begin{table*}
\footnotesize
\centering
\caption{Evalutation on CIFAR10/100 using Wide-ResNet-28-10 \cite{zagoruyko2016wide} with Monte Carlo dropout \cite{gal2016dropout}. Combining LATES with MC Dropout matches previously-reported state-of-the-art results \cite{nado2021uncertainty}.}
\label{table:mc-cifar}
\begin{tabularx}{6in}{@{}cclXXXXX@{}}
\toprule
   &Dropout
   &
   &
  ECE $\downarrow$ &
  NLL $\downarrow$&
  Brier $\downarrow$&
  ACC $\uparrow$ &
  AUC $\uparrow$ \\ \midrule

  {\centering CIFAR10 }&0.1&
  \begin{tabular}[c]{@{}l@{}}\phantom{+} MCD\\+ \em LATES\end{tabular} &
  \begin{tabular}[c]{@{}c@{}}0.023  $\pm$ 0.001\\ {\bf0.010 $\pm$ 0.001}\end{tabular} &
  \begin{tabular}[c]{@{}c@{}}0.160 $\pm$ 0.004\\ {\bf0.144 $\pm$ 0.003}\end{tabular} &
  \begin{tabular}[c]{@{}c@{}}-0.932 $\pm$ 0.001\\ {\bf-0.935 $\pm$ 0.001}\end{tabular} &
  \begin{tabular}[c]{@{}c@{}}{\bf0.958 $\pm$ 0.001}\\ 0.957 $\pm$ 0.001\end{tabular} &
  \begin{tabular}[c]{@{}c@{}}0.995 $\pm$ 0.000\\ {\bf0.997 $\pm$ 0.000}\end{tabular} \\ \\
 &
  0.2 &
  \begin{tabular}[c]{@{}l@{}}\phantom{+} MCD\\+ \em LATES\end{tabular} &
  \begin{tabular}[c]{@{}c@{}}0.019 $\pm$ 0.001\\ {\bf0.0045 $\pm$ 0.000}\end{tabular} &
  \begin{tabular}[c]{@{}c@{}}0.151 $\pm$ 0.002\\ {\bf0.144 $\pm$ 0.001}\end{tabular} &
  \begin{tabular}[c]{@{}c@{}}-0.929 $\pm$ 0.001\\ {\bf-0.930 $\pm$ 0.000}\end{tabular} &
  \begin{tabular}[c]{@{}c@{}}{\bf0.954 $\pm$ 0.001}\\ 0.953 $\pm$ 0.001\end{tabular} &
  \begin{tabular}[c]{@{}c@{}}0.996 $\pm$ 0.000\\ {\bf0.998 $\pm$ 0.000}\end{tabular} \\ \midrule

CIFAR100 &
  0.1 &
  \begin{tabular}[c]{@{}l@{}}\phantom{+} MCD\\ \em+ LATES\end{tabular} &
  \begin{tabular}[c]{@{}c@{}}0.044 $\pm$ 0.002\\ {\bf0.033 $\pm$ 0.002}\end{tabular} &
  \begin{tabular}[c]{@{}c@{}}0.754  $\pm$ 0.002\\ {\bf0.751 $\pm$ 0.002}\end{tabular} &
  \begin{tabular}[c]{@{}c@{}}-0.717 $\pm$ 0.002\\ {\bf-0.719 $\pm$ 0.002}\end{tabular} &
  \begin{tabular}[c]{@{}c@{}}{\bf0.802 $\pm$ 0.002}\\ {\bf0.802 $\pm$ 0.002}\end{tabular} &
  \begin{tabular}[c]{@{}c@{}}0.985 $\pm$ 0.000\\ {\bf0.988 $\pm$ 0.000}\end{tabular}  \\ \\
 &
  0.2 &
  \begin{tabular}[c]{@{}l@{}}\phantom{+} MCD\\ \em+ LATES\end{tabular} &
  \begin{tabular}[c]{@{}c@{}}0.035 $\pm$ 0.003\\ {\bf0.016 $\pm$ 0.002}\end{tabular} &
  \begin{tabular}[c]{@{}c@{}}0.709 $\pm$ 0.005\\ {\bf0.716 $\pm$ 0.004}\end{tabular} &
  \begin{tabular}[c]{@{}c@{}}{\bf-0.718 $\pm$ 0.002}\\ -0.716 $\pm$ 0.002\end{tabular} &
  \begin{tabular}[c]{@{}c@{}}{\bf0.798 $\pm$ 0.002}\\ 0.794 $\pm$ 0.002\end{tabular} &

  \begin{tabular}[c]{@{}c@{}}{0.988 $\pm$ 0.000}\\ \bf{0.991 $\pm$ 0.001}\end{tabular}
\\\bottomrule
  
\end{tabularx}
\end{table*}

\begin{figure*}
    \centering
    \includegraphics[width=2\columnwidth]{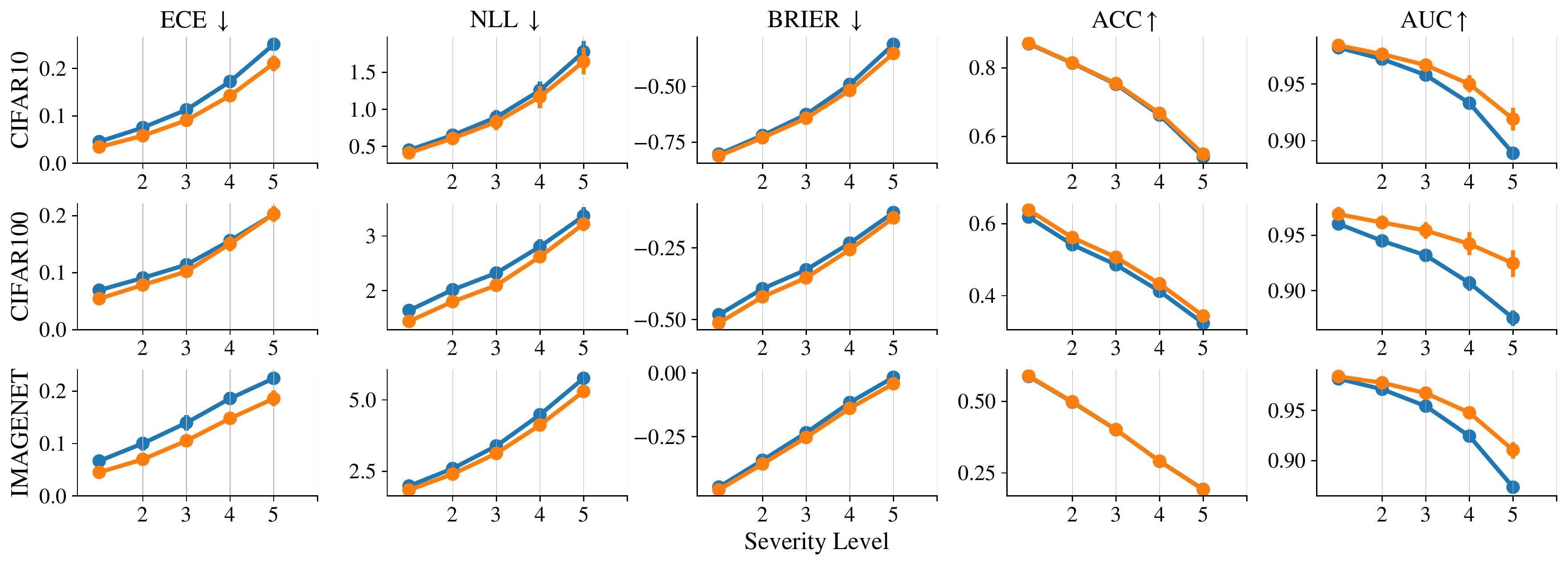}
    \caption{A summary of performance results for OOD data, averaged across corruption types and architectures. LATES (shown in {\color{brown}\bf amber}) yields an improvement over temperature scaling (shown in {\color{blue}\bf blue} in all performance metrics, particularly for large distribution shifts.}
    \label{fig:ood_curves}
\end{figure*}

\begin{figure}
    \centering
    \includegraphics[width=\columnwidth]{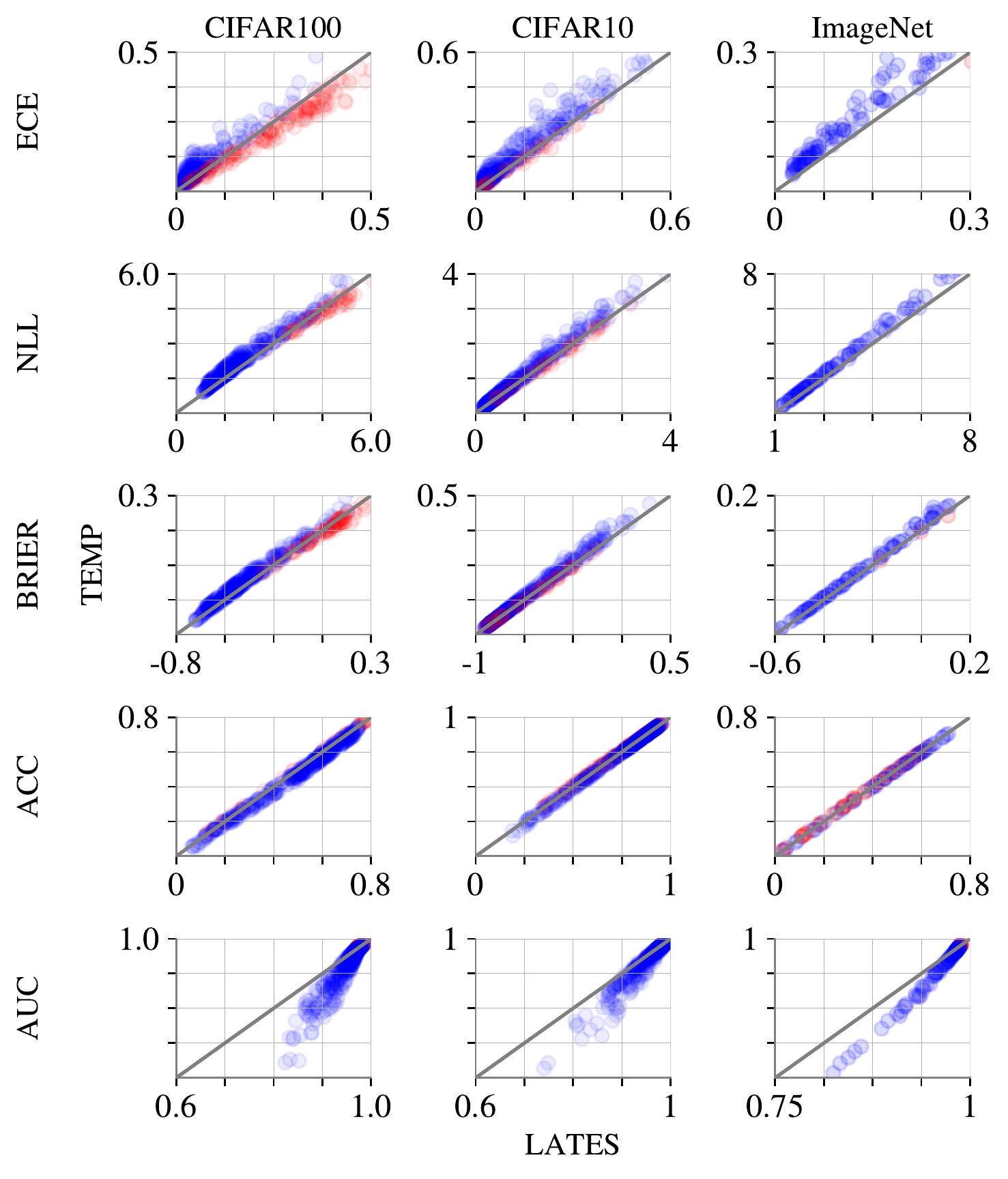}
    \caption{A scatter plot summarizing all experiments on OOD data for CIFAR100, CIFAR10, and ImageNet datasets. A blue marker indicates an experiment where LATES outperforms temperature scaling. Darker regions indicate more experiments. In general, LATES leads to an improvement across all metrics, but the improvement is particularly notable for AUC. One  exception is improving accuracy in ImageNet, in which LATES seems to have little effect. See Table \ref{table:stat} for statistical significance tests.}
    \label{fig:ood_scatter}
\end{figure}

\subsection{Out-of-Distribution Data}
Next, we evaluate LATES in the OOD setting, where various types of corruptions are added to the test data. For that, we follow the corruption types and five severity levels available in the Uncertainty Baselines Benchmark~\cite{nado2021uncertainty}. Figure \ref{fig:ood_curves} summarizes the results. We observe, again, an improvement in LATES compared to temperature scaling, particularly for large distribution shifts (i.e. when the corruption severity level is large). In particular, the improvement in LATES with respect to AUC is quite significant in all three datasets.

\subsection{Statistical Analysis}\label{sect:stat}
To verify statistical significance, we use  the non-parametric Wilcoxon signed rank test for comparing  LATES with temperature scaling \cite{demvsar2006statistical}. We use all corruption types at all severity levels across all architectures and random seeds.
Table \ref{table:stat} summarizes the main results. We observe that the improvement in LATES compared to temperature scaling is statistically significant at the 99.9\% confidence level in all metrics and datasets; nearly all $p$ values are below $10^{-9}$. The only exception is accuracy in ImageNet, in which there is no statistically significant difference. However, as shown in Figure \ref{fig:ood_scatter}, LATES improves AUC in ImageNet significantly under distribution shift. In Table \ref{table:anova}, we provide a different analysis of statistical significance using the Analysis of Variance (ANOVA). 

\begin{table}
\scriptsize

\caption{A summary of $p$-values using the non-parametric Wilcoxon signed rank test for comparing  LATES with temperature scaling. All improvements in LATES are statistically significant with $p$-values less than $10^{-9}$, with the exception of improving accuracy in ImageNet in which there is no statistically significant difference. Since the $p$ values are quite small, statistical significance remains valid even after correcting for multiple hypothesis testing using Holm's step down procedure \cite{demvsar2006statistical}.}\vspace{1mm}\label{table:stat}
\begin{tabularx}{\columnwidth}{@{}lXXXXX@{\hspace{0.1em}}}
\toprule
Dataset
& ECE & NLL & Brier & ACC & AUC
    \\ \midrule
 CIFAR100 &  $<10^{-9}$
    & $<10^{-9}$ & $<10^{-9}$ &$<10^{-9}$ & $<10^{-9}$  \\ 
 CIFAR10 &  $<10^{-9}$
    & $<10^{-9}$ & $<10^{-9}$ &$<10^{-9}$ & $<10^{-9}$  \\ 
 ImageNet &  $<10^{-9}$
    & $<10^{-9}$ & $<10^{-9}$ &$0.18$ & $<10^{-9}$  \\ \bottomrule
\end{tabularx}
\end{table}

\subsection{Complementing Monte Carlo Dropout}
LATES can also be combined with other existing techniques to improve their performance further. We illustrate this using the popular Monte Carlo Dropout \cite{gal2016dropout}, in which the method matches with state-of-the-art results as reported by the Uncertainty Baselines Benchmark~\cite{nado2021uncertainty}.

In Monte Carlo dropout, multiple predictions are generated for the same instance $\xb\in\mathcal{X}$ by adding dropout layers to the model and using them at \emph{both} training and test time. Then, the model's uncertainty can be estimated by averaging the predictions.
In our experiments, we use the implementation of the Monte Carlo dropout available at the Uncertainty Baselines Benchmark~\cite{nado2021uncertainty}. We use the same implementation details and hyper-parameters suggested by the benchmark, which uses the Wide-ResNet-28-10 architecture, and combine it with LATES. We also test on different values of the dropout rate. Final results are in Table~\ref{table:mc-cifar}. 

The state-of-the-art results on CIFAR10 reported in the Uncertainty Baseline Benchmark uses the Spectral Normalized Neural Gaussian Processes (SNGP)~\cite{liu2020simple} when combined with Augmentation Mixup (AugMix)~\cite{hendrycks2019augmix} and achieves ECE of 0.0045 at 96\% accuracy. We achieve the same ECE at 95.7\% accuracy with very limited data augmentation (only left/right flipping). On CIFAR100 we achieve near SOTA ECE with much less parameters compared to the best methods that use an ensemble of SNGP models combined with Monte Carlo Dropout. However when comparing with the same number of parameters, we achieve less ECE for the same accuracy. Specifically, our ECE in CIFAR100 of 0.016 with a 35M-parameter model is below the best ECE of 0.017 reported in the benchmark that uses a comparable model size of 35M parameters \cite{nado2021uncertainty}. The accuracy is 79\% in both cases.

\section{Conclusion}
In this paper, we introduce layer-stack temperature scaling (LATES); a simple, effective technique to improve both the calibration and accuracy of deep neural networks. It is inspired by several lines of work, which relate example difficulty to the activations at the intermediate layers of the neural network. We motivate LATES theoretically and demonstrate that it leads to a significant improvement over temperature scaling both in- and out-of-distribution, especially in the low-data regime. In particular, LATES provides a significant improvement in AUC under large distribution shifts. We support these conclusions by a comprehensive statistical analysis. Finally, we demonstrate that LATES can match state-of-the-art results on CIFAR10 and CIFAR100 when combined with Monte Carlo Dropout.

{\small
\bibliographystyle{ieee_fullname}
\bibliography{egbib}
}

\appendix
\end{document}